\documentclass[pmlr]{jmlr}

\usepackage{longtable}

\usepackage{booktabs}
\usepackage[load-configurations=version-1]{siunitx} 

\makeatletter
\def\set@curr@file#1{\def\@curr@file{#1}} 
\makeatother

\theorembodyfont{\upshape}
\theoremheaderfont{\scshape}
\theorempostheader{:}
\theoremsep{\newline}

\usepackage{algorithm}
\usepackage{algpseudocode}
\usepackage{algcompatible}

\usepackage{enumitem}
\usepackage{algorithm}
\usepackage{algpseudocode}
\usepackage{algcompatible}
\usepackage{amsmath,amsfonts,amssymb,mathtools}
\newtheorem{thm}{Theorem}[section]
\newtheorem{defin}{Definition}[section]
\newtheorem{lem}{Lemma}[section]
\usepackage{multirow}

\title[ Differentially-private Federated Neural Architecture Search]{ Differentially-private Federated Neural Architecture Search}

\author{\Name{Ishika Singh\textsuperscript{*$\dagger$}}\\
       \addr University of California San Diego 
\AND
\Name{Haoyi Zhou\textsuperscript{*$\dagger$}}
\\
       \addr
Rutgers University 
\AND
\Name{Kunlin Yang}\\
       \addr University of California San Diego
\AND
 \Name{Meng Ding\textsuperscript{$\dagger$}}\\
       \addr University of California San Diego
       \AND
        \Name{Bill Lin}\\
       \addr University of California San Diego
       \AND
       \Name{Pengtao Xie}
       \Email{pengtaoxie2008@gmail.com}\\
       \addr University of California San Diego
\AND
       }

\begin{document}

\maketitle

\begin{abstract}
Neural architecture search, which aims to automatically search for architectures (e.g., convolution, max pooling) of neural networks that maximize validation performance, has achieved remarkable progress recently. In many application scenarios, several parties would like to collaboratively search for a shared neural architecture by leveraging data from all parties. However, due to privacy concerns, no party wants its data to be seen by other parties. To address this problem, we propose federated neural architecture search (FNAS), where different parties collectively search for a differentiable architecture by exchanging gradients of architecture variables without exposing their data to other parties. To further preserve privacy, we study differentially-private FNAS (DP-FNAS), which adds random noise to the gradients of architecture variables. We provide theoretical guarantees of DP-FNAS in achieving differential privacy. Experiments show that DP-FNAS can search highly-performant neural architectures while protecting the privacy of individual parties. The code is available at \url{https://github.com/UCSD-AI4H/DP-FNAS}
\end{abstract}

\section{Introduction}

\let\thefootnote\relax\footnotetext{*Equal contribution}
\let\thefootnote\relax\footnotetext{$\dagger$The work was done during internship at UCSD.}

In many application scenarios, the data owner would like to train machine learning (ML) models using their data that contains sensitive information, but the size of the data is limited. Many ML methods, especially deep learning methods, are data hungry. Having more data for model training usually improves performance. One way to have more training data is to combine data of the same kind from multiple parties and use the combined data to collectively train a model. However, since each of these datasets contains private information, they are not allowed to share across parties. Federated learning~\citep{konevcny2016federated,mcmahan2016communication} is developed to address this problem. Multiple parties collectively train a shared model in a decentralized way by exchanging sufficient statistics (e.g., gradients) without exposing the data of one party to another. 

While preserving privacy by avoiding sharing data among different parties, federated learning (FL) incurs difficulty for model design. When ML experts design the model architecture, they need to thoroughly analyze the properties of data to obtain insights that are crucial in determining which architecture to use. In an FL setting, an expert from one party can only see the data from this party and is not able to analyze the data from other parties. Without having a global picture of all data from different parties, ML experts are not well-equipped to design a model architecture that is optimal for fulfilling the predictive tasks in all parties. To address this problem, we resort to automated neural architecture search~\citep{zoph2016neural,liu2018darts,real2019regularized} by designing search algorithms to automatically find out the optimal architecture that yields the best performance on the validation datasets.

To this end, we study federated neural architecture search (FNAS), where multiple parties collaboratively search for an optimal neural architecture without exchanging sensitive data with each other for the sake of preserving privacy. For computational efficiency, we adopt a differentiable search strategy~\citep{liu2018darts}. The search space is overparameterized by a large set of candidate operations (e.g., convolution, max pooling) applied to intermediate representations (e.g., feature maps in CNN). Each operation is associated with an architecture variable indicating how important this operation is. The prediction loss is a continuous function w.r.t the architecture variables $A$ as well as the weight parameters $W$ in individual operations. $A$ and $W$ are learned by minimizing the validation loss using a gradient descent algorithm. After learning, operations with top-$K$ largest architecture-variables are retained to form the final architecture. 

In FNAS, a server maintains the global state of $A$ and $W$. Each party has a local copy of $A$ and $W$. In each iteration of the search algorithm, each party calculates gradient updates of $A$ and $W$ based on its local data and local parameter copy, then sends the gradients to the server. The server aggregates the gradients received from different parties, performs a gradient descent update of the global state of $A$ and $W$, and sends the updated parameters back to each party, which replaces its local copy with the newly received global  parameters. This procedure iterates until convergence. 

Avoiding exposing data is not sufficient for privacy preservation. Several studies~\citep{bhowmick2018protection,carlini2018secret,fredrikson2015model} have shown that intermediates results such as gradients can reveal private information. To address this problem, we study differentially-private FNAS (DP-FNAS), which adds random noise to the gradients calculated by each party to retain differential privacy~\citep{dwork2006calibrating}. We provide theoretical guarantees of DP-FNAS in privacy preservation. Experiments demonstrate that while protecting the privacy of individual parties, the architectures searched by DP-FNAS can achieve high accuracy that is comparable to those searched by single-party NAS.

The major contributions of this paper are as follows:
\begin{itemize}[leftmargin=*]
\item We propose differentially-private  federated learning for neural architecture search (DP-FNAS), which enables multiple parties to collaboratively search for a highly-performant neural architecture without sacrificing privacy. 
\item We propose a DP-FNAS algorithm which uses a parameter server framework and a gradient-based method to perform federated search of neural architectures. The gradient is obfuscated with random noise to achieve differential privacy.
\item  We provide a theoretical guarantee of our algorithm in terms of privacy preservation. 
\item We perform experiments which show that DP-FNAS can search highly-performant neural architectures while protecting the privacy of individual parties.. 
\end{itemize}

The rest of the papers are organized as follows. Section 2 and 3 present the method and experiments. Section 4 reviews related works and Section 5 concludes the paper.

\section{Methods}

We assume there are $K$ parties aiming to solve the same predictive task, e.g., predicting whether a patient has pneumonia based on his or her chest X-ray image. Each party $k$ has a  labeled dataset $D_k$ containing pairs of input data example and its label. For instance, the data example could be a chest X-ray and the label is about whether the patient has pneumonia. The datasets contain sensitive information where privacy needs to be strongly protected. Therefore, the $K$ parities cannot share their datasets with each other. One naive approach is: each party trains a model using its own data. However, deep learning methods are data hungry: more training data usually leads to better predictive performance. It is preferable to leverage all datasets from different parties to collectively train a model, which presumably has better predictive performance than the individual models belonging-to different parties, each trained on a party-specific dataset. How can one achieve this goal without sharing data between parties?  

Federated learning (FL)~\citep{konevcny2016federated,mcmahan2016communication} is a learning paradigm designed to address this challenge. In FL, different parties collectively train a model by exchanging sufficient statistics (e.g., gradient) calculated from their datasets, instead of exchanging the original data directly. There is a server maintaining the weight parameters of the global model to be trained. Each party has a local copy of the model. In each iteration of the training algorithm, each party $k$ uses its data $D_k$ and the local model $M_k$ to calculate a gradient $G_k$ of the predictive loss function $L(M_k,D_k)$ with respect to $M_k$. Then it sends $G_k$ to the server. The server aggregates the gradients $\{G_k\}_{k=1}^K$ received from different workers and performs a gradient descent update of the global model: $M\gets M-\eta \frac{1}{K}\sum_{k=1}^M G_k$, where $\eta$ is the learning rate. Then it sends the updated global model back to each party, which replaces its local model with the global one. This procedure iterates until convergence. In this process, the dataset of each party is not exposed to any other party or the server.  Hence its privacy can be protected to some extent (later, we will discuss a stronger way of protecting privacy).

Though FL provides a nice way of effectively using more data for model training while preserving privacy, it poses some difficulties on how to design the model architecture. For ML experts, to design an effective model architecture, the experts need to thoroughly analyze the properties of the data. In an FL setting, the expert from each party can only see the data from this party, not that from others. Without a global picture of all datasets, these experts are not well-equipped to design an architecture that is optimal for the tasks in all parties.

To address this problem, we resort to automatic neural architecture search (NAS)~\citep{zoph2016neural,liu2018darts,real2019regularized}. Given a predictive task and labeled data, NAS aims to automatically search for the optimal neural architecture that can best fulfill the targeted task. The problem can be formulated in the following way:
\begin{equation}
\begin{array}{ll}
    \textrm{min}_A & L(D^{\textrm{(val)}},A, W^*(A)) \\
    s.t. &  W^*(A)=\textrm{argmin}_{W}\;\;L(D^{\textrm{(tr)}},A, W)
\end{array}
\end{equation}
where $D^{\textrm{(tr)}}$ and $D^{\textrm{(val)}}$ are the training data and validation data respectively. $A$ denotes the neural architecture and $W$ denotes the weights of the model whose architecture is $A$. Given a configuration $A$ of the architecture, we train it on the training data and obtain the best weights $W^*(A)$. Then we measure the loss $L(D^{\textrm{(val)}},A, W^*(A))$ of the trained model on the validation set. The goal of an NAS algorithm is to identify the best $A$ that yields the lowest validation loss. Existing search algorithms are mostly based on reinforcement learning~\citep{zoph2016neural}, evolutionary algorithm~\citep{real2019regularized}, and differentiable NAS~\citep{liu2018darts}. In this work, we focus on differentiable NAS since it is computationally efficient. 

To this end, we introduce federated neural architecture search (FNAS), which aims to leverage the datasets from different parties to collectively learn a neural architecture that can optimally perform the predictive task, without sharing privacy-sensitive data between these parties. The FNAS problem can be formulated as:
\begin{equation}
\label{eq:fnas}
\begin{array}{ll}
    \textrm{min}_A & \sum_{k=1}^K L(D_k^{\textrm{(val)}},A, W^*(A)) \\
    s.t. &  W^*(A)=\textrm{argmin}_{W}\;\; \sum_{k=1}^K L(D_k^{\textrm{(tr)}},A, W)
\end{array}
\end{equation}
where $D_k^{\textrm{(tr)}}$ and $D_k^{\textrm{(val)}}$ denote the training and validation dataset belonging to the party $k$ respectively. A naive algorithm for FNAS performs the following steps iteratively: given a configuration $A$ of the architecture, use the gradient-based FL method to learn the optimal weights $W^*(A)$ on the training data; then evaluate $W^*(A)$ on the validation data of each party and aggregate the evaluation results. The validation performance is used to select the best architecture. Certainly, this is not efficient or scalable. We resort to a differentiable search approach~\citep{liu2018darts}. The basic idea of differentiable NAS is: set up an overparameterized network that combines many different types of operations; each operation  is associated with an architecture variable (AV) indicating how important the operation is; optimize these AVs together with the weight parameters in the operations to achieve the best performance on the validation set; operations with top-$K$ largest AVs are selected to form the final architecture. A neural architecture can be represented as a directed acyclic graph (DAG) where the nodes represent intermediate representations (e.g., feature maps in CNN) and edges represent  operations (e.g., convolution, pooling) over nodes. Each node $x_i$ is calculated in the following way: $x_i=\sum_{j\in \mathcal{P}_i} e_{ji} (x_j)$, where $\mathcal{P}_i$ is a set containing the ancestor nodes of $i$. $e_{ji}(\cdot)$ denotes the operation associated with the edge connecting $j$ to $i$. In differentiable NAS, this DAG is overparameterized: the operation $e_{ji}(\cdot)$ on each edge is a weighted combination of all possible operations. Namely, $e_{ji}(x)=\sum_{m=1}^{M} \frac{\exp(a_{jim})}{\sum_{l=1}^{K}\exp(a_{jil})}o_m(x)$, where $o_m(\cdot)$ is the $m$-th operation (parameterized by a set of weights) and $M$ is the total number of operations. $a_{jim}$ is an architecture variable representing how important $o_m(\cdot)$ is.  In the end, the prediction function of this neural network is a continuous one parameterized by the variables $A=\{a\}$ representing the architecture and the weight parameters $W$. The prediction loss function is end-to-end differentiable w.r.t both $A$ and $W$, which can be learned by gradient descent. After learning, operations with top-$K$ largest architecture variables are retained to form the final architecture. The problem in Eq.(\ref{eq:fnas}) can be approximately solved by iteratively performing the following two steps:
\begin{itemize}[leftmargin=*]
\item Update weight parameters $W$:
\begin{equation}
 W\gets W-\xi \sum_{k=1}^K \nabla_{W} L(D_k^{\textrm{(tr)}},A, W)
\end{equation}
    \item Update architecture variables $A$: 
\begin{equation}
    A\gets A-\eta \sum_{k=1}^K \nabla_{A} L(D_k^{\textrm{(val)}},A, W-\xi \sum_{j=1}^K \nabla_{W} L(D_j^{\textrm{(tr)}},A, W))
\end{equation}
where $\nabla_{A} L(D_k^{\textrm{(val)}},A, W-\xi \sum_{j=1}^K \nabla_{W} L(D_j^{\textrm{(tr)}},A, W))$ can be approximately computed as
\begin{equation}
\label{eq:hk}
  H_k= \nabla_{A} L(D_k^{\textrm{(val)}},A, W’)-\frac{\xi}{2\epsilon} (\nabla_{A}  L(D_k^{\textrm{(tr)}},A, W^+ )-\nabla_{A}  L(D_k^{\textrm{(tr)}},A, W^-))
\end{equation}
where $W'=W-\xi \sum_{j=1}^K \nabla_{W} L(D_j^{\textrm{(tr)}},A, W)$, $W^+=W+\epsilon \nabla_{W'} L(D_k^{\textrm{(val)}},A, W’)$, and $W^-=W-\epsilon \nabla_{W'} L(D_k^{\textrm{(val)}},A, W’)$.  
\end{itemize}

The server holds the global version of $A$ and $W$. Each party $k$ has a local copy: $A_k$ and $W_k$, and also holds an auxiliary variable $W'_k$. FNAS iteratively performs the following steps until convergence. (1) Each party uses $A_k$, $W_k$, and $D_k^{\textrm{(tr)}}$ to calculate $\nabla_{W_k} L(D_k^{\textrm{(tr)}},A_k, W_k)$, and sends it to the server; (2) The server aggregates $\{\nabla_{W_k} L(D_k^{\textrm{(tr)}},A_k, W_k)\}_{k=1}^K$ received from different parties, performs a gradient descent update of the global $W$: $W\gets W-\xi \sum_{k=1}^K \nabla_{W_k} L(D_k^{\textrm{(tr)}},A_k, W_k)$, and sends the updated global $W$ to each party which replaces its $W'_k$ with $W$; (3) Each party calculates the gradient $H_k$ in Eq.(\ref{eq:hk}) and sends it to the server; (4) The server aggregates $\{H_k\}_{k=1}^K$ received from different parties, updates $A\gets A-\eta \sum_{k=1}^K H_k$, and sends the updated $A$ to each party; (4) Each party replaces $A_k$ with $A$ and replaces $W_k$ with $W'_k$.

In federated NAS, while the sensitive data of each party can be protected to some extent by avoiding sharing the data with other parties, there is still a significant risk of leaking privacy due to the sharing of intermediate sufficient statistics (e.g., gradients) among parties. It has been shown in several works that the intermediate sufficient statistics can reveal private information if leveraged cleverly~\citep{bhowmick2018protection,carlini2018secret,fredrikson2015model}. To address this problem, we study differentially-private (DP) FNAS, which uses DP techniques~\citep{dwork2006calibrating,dwork2008differential} to achieve a stronger preservation of privacy. A DP algorithm (with a parameter $\alpha$ measuring the strength of privacy protection) guarantees
that the log-likelihood ratio of the outputs of the algorithm
under two databases differing in a single individual’s data is
smaller than $\alpha$. That means, regardless of whether the individual is present in the data, an adversary’s inferences about this individual will be similar if $\alpha$ is small enough. Therefore, the privacy of this individual can be strongly protected. Several works have shown that adding random noise to the gradient can achieve differential privacy~\citep{rajkumar2012differentially,song2013stochastic,agarwal2018cpsgd}. In this work, we follow the same strategy. For each worker, the gradient updates of $A$ and $W$ are added with random Gaussian noise before  sent to the server:
\begin{equation}
\label{eq:dp_gk}
        G_k=\nabla_{W_k} L(D_k^{\textrm{(tr)}},A_k, W_k)+U_k
\end{equation}
\begin{equation}
\label{eq:dp_hk}
        H_k=\nabla_{A_k} L(D_k^{\textrm{(val)}},A_k, W_k’)-\frac{\xi}{2\epsilon} (\nabla_{A_k}  L(D_k^{\textrm{(tr)}},A_k, W_k^+ )-\nabla_{A_k}  L(D_k^{\textrm{(tr)}},A_k, W_k^-))+V_k
\end{equation}
where the elements of $U$ and $V$ are drawn randomly from univariate Gaussian distributions with zero mean and a variance of   $\sigma^2_k$ and $\gamma^2_k$ respectively. Algorithm~\ref{fig:cov-2m}  shows the execution workflow in one iteration of the differentially-private federated NAS (DP-FNAS) algorithm. Per-sample gradient clipping is used with hyperparameters $R^G$ and $R^H$.

\begin{algorithm}
\begin{algorithmic}
    \FOR{each party $k$}
        \STATE{Take a Poisson subsample $I_t \subseteq \{1, . . . , N_k^{(tr)}\}$ with subsampling probability $p$}
        \FOR{$i \in I_t$}
            \STATE{$g_{t}^{(i)}=\nabla_{W_{k}} L\left(D_{k}^{(\mathrm{tr})(i)}, A_{k}, W_{k}\right)$}
            \STATE{$ \bar{g}_{t}^{(i)} = g_{t}^{(i)} / \max \left\{1,\left\|g_{t}^{(i)}\right\|_{2} / R^G\right\}$} \algorithmiccomment{\textit{Gradient clipping}}
        \ENDFOR
        \STATE{$ G_k =\frac{1}{\left|I_{t}\right|}\left(\sum_{i \in I_{t}} \bar{g}_{t}^{(i)}+ R^G.U_k \right)$} \algorithmiccomment{\textit{Gaussian mechanism}}
    \ENDFOR
    \newline
    \STATE{On the server side:}
    \STATE{Update $W \leftarrow W-\epsilon \sum_{k=1}^{K} G_{k}$}
    \STATE{Send $W$ to each party}
    \newline
    \FOR{each party $k$}
        \STATE{Update $W_k^{'} \leftarrow W$}
        \STATE{Take a Poisson subsample $I_t \subseteq \{1, . . . , N_k^{(val)}\}$ with subsampling probability $p$}
        \FOR{$i \in I_t$}
            \STATE{$h_{t}^{(i)}=\nabla_{A} L\left(D_{k}^{(\mathrm{val})(i)}, A, W\right)$} 
            \STATE{$ \bar{h}_{t}^{(i)} = h_{t}^{(i)} / \max \left\{1,\left\|h_{t}^{(i)}\right\|_{2} / R^H\right\}$} \algorithmiccomment{\textit{Gradient clipping}}
        \ENDFOR
        \STATE{$ H_k =\frac{1}{\left|I_{t}\right|}\left(\sum_{i \in I_{t}} \bar{h}_{t}^{(i)}+ R^H.V_k \right)$} \algorithmiccomment{\textit{Gaussian mechanism}}
    \ENDFOR
    \newline
    \STATE{On the server side:}
    \STATE{Update $A \leftarrow A-\eta \sum_{k=1}^{K} H_{k}$}
    \STATE{Send $A$ to each party}
    \newline
    \FOR{each party k}
        \STATE{Update $A_k \leftarrow A$}
        \STATE{Update $W_k \leftarrow W_k'$}
    \ENDFOR
\end{algorithmic}
\caption{Execution semantics in each iteration of the DP-FNAS algorithm}
\label{fig:cov-2m}
\end{algorithm}


\section{Theoretical Analysis}
In this section, we provide theoretical analysis on the differential privacy (DP)  guarantees of the proposed DP-FNAS algorithm. We consider a recently proposed privacy definition, named $f$-DP~\citep{dong2019gaussian} owing to its tractable and lossless handling of privacy primitives like composition, subsampling, etc. and superior accuracy results than $(\epsilon, \delta)$-DP~\citep{dong2019gaussian, bu2019deep}. Broadly, composition is concerned with a sequence of analysis on the same dataset where each analysis is informed by the exploration of prior analysis from the previous iteration. Our proposed gradient-based FNAS algorithm involves two instances of private gradient sharing or Gaussian mechanism~\citep{10.1561/0400000042}, for optimizing weight parameters and architecture variables, between the parties and the central server. One of the two mechanisms composes over the other in one iteration, hence they keep composing onto each other over further iterations of the algorithm. We provide a decoupling analysis of these two mechanisms over the iterations, by leveraging the fact that the datasets used for the two mechanisms are disjoint (one on training set, the other on validation set). 
We get the results in terms of Gaussian differential privacy (the focal point of the $f$-DP guarantee family), which ensure privacy in a very interpretable manner by associating it to the hardness of telling apart two shifted normal distributions.

$f$-DP is a relaxation of $(\epsilon, \delta)$-DP recently proposed by \citep{dong2019gaussian}. This new privacy definition follows directly from the hypothesis testing understanding of differential privacy. Moreover, it can efficiently analyze common primitives contributing towards differential privacy, including composition, and privacy amplification by subsampling. In our proposed FNAS algorithm, mini-batch subsampling is used for improving computational efficiency. 
 A side benefit of subsampling is that it naturally offers tighter privacy bounds since an individual not contained in a subsampled mini-batch enjoys perfect privacy. The $f$-DP leverages this fact efficiently for amplifying privacy. In addition, $f$-DP encompasses a canonical single-parameter family, called Gaussian differential privacy (GDP). GDP is the focal privacy definition, due to a central limit theorem, stating that the privacy guarantee of the composition of private algorithms are approximately equivalent to differentiating between two shifted normal distributions.

\subsection{Preliminaries}
An algorithm conserves privacy if it is hard for the adversary to find out information about a particular individual from the shared statistics across parties. We usually study this problem in the context of neighboring datasets, that is, how difficult it is for a third party to determine presence of any individual data point in the shared statistics across two neighbouring datasets. Let $S$ and $S'$ be any two neighboring datasets, then one can derived by adding an individual to the other or vice versa. Let the shared statistics be two probability distributions associated with $\mathcal{M}(S)$ and $\mathcal{M}(S')$, where $\mathcal{M}$ is a randomized mechanism. The adversary wants to differentiate these two distributions in order to identify the differential data point and hence seeks to test the following two simple hypothesis: $H0$ : the true dataset is $S$,  versus $H1$ : the true dataset is $S'$. The true dataset here being the one with the differential data point. If this hypothesis testing problem is hard enough, the privacy will be preserved well. The 
$(\epsilon, \delta)$-DP definition~\citep{dwork2008differential} essentially utilizes the worst-case likelihood ratio of the distributions associated with $\mathcal{M}(S)$ and $\mathcal{M}(S')$ for measuring this hardness. While $f$-DP comes with a more intuitive and direct measure of this hardness, by directly dealing with the trade-off function associated with the hypothesis testing. It uses the trade-off between type I ($\mathbb{E}_{P}$) and type II ($1-\mathbb{E}_{Q}$) error, where $P$ and $Q$ are the distributions associated with the above two mechanisms. With these in place, we present the following definitions from \citep{dong2019gaussian} for our proof.

\begin{defin} (Trade-off Function) Let $P$ and $Q$ denote the distributions of $\mathcal{M}(S)$ and $\mathcal{M}(S')$, respectively, and let $\phi$ be any (possibly randomized) rejection rule for testing $H0$ : $P$ against $H1$ : $Q$. The trade-off function of P and Q is defined as:
\begin{align*}
     T(P, Q):[0,1] & \mapsto[0,1] \\ \alpha & \mapsto \inf _{\phi}\left\{1-\mathbb{E}_{Q}[\phi]: \mathbb{E}_{P}[\phi] \leqslant \alpha\right\}
\end{align*}
\end{defin}

\begin{defin}
Let $G_{\mu}:=T(\mathcal{N}(0,1), \mathcal{N}(\mu, 1)) \text { for } \mu \geqslant 0$. A (randomized) algorithm $\mathcal{M}$ is $\mu$-Gaussian differentially private (GDP) if $T\left(M(S), M\left(S^{\prime}\right)\right) \geqslant G_{\mu}$, for all neighboring datasets $S$ and $S'$.
\end{defin}
That is, $\mu$-GDP says that determining whether any individual is in the given dataset is at least as involved as differentiating between the two normal distributions $\mathcal{N}(0,1)$ and  $\mathcal{N}(\mu, 1)$ based on one draw.

\subsection{Privacy analysis}
The major results are summarized in the following theorem.
\begin{thm}
\label{thm:dp}
Consider a gradient-based Federated NAS algorithm (Algorithm \ref{fig:cov-2m}), which subsamples minibatches (using Poisson subsampling), clips gradients, and perturbs gradients for both weight parameters $W$ and architecture variables $A$ using Gaussian mechanism $\mathcal{M}_t$ at each iteration. 
Assuming that $D_k^{(tr)}$ and $D_k^{(val)}$ are disjoint for each party $k$, the algorithm achieves 
\begin{align*}
    \frac{B}{N_k^{(tr)}} \sqrt{T\left(\mathrm{e}^{1 / \sigma^{2}-1}\right)}&\text{-GDP for mechanism composition } \mathcal{M}_{t=1:T}^{G_k}(D_k^{(tr)}) \text{ and}\\ \frac{B}{N_k^{(val)}} \sqrt{T\left(\mathrm{e}^{1 / \tau^{2}-1}\right)}&\text{-GDP for mechanism composition } \mathcal{M}_{t=1:T}^{H_k}(D_k^{(val)})
\end{align*}
where GDP refers to Gaussian Differential Privacy, $\sigma ^2$ and $\tau ^2$ represent the variance of the added Gaussian noises $U_k$ and $V_k$ respectively, $T$ is the number of iterations, $B$ is the mini-batch size, $N_k^{(tr)}$ and $N_k^{(val)}$ are the number of training and validation examples owned by party $k$, respectively. 
\end{thm}
\textbf{Remarks}
\begin{itemize}[leftmargin=*]
    \item The privacy bounds suggest that the algorithm achieves better privacy guarantees if $B\sqrt{T} /N_k$ is small, and $\sigma$ or $\tau$ are not very small.
    \item Since GDP is achieved through central limit theorem due to composition of distributions $\mathcal{M}_t(D)$ over $T$ iterations, it is expected that $T$ is large enough. This requirement is usually satisfied with general settings of DP-FNAS training procedure. 
    \item We can also choose different subsampling probability for the two processes, which will reflect accordingly in the privacy bound ($p = B/N_k$). We may also use other subsampling methods like shuffling (randomly permuting and dividing data into folds at each epoch) and uniform sampling (sampling a batch of size $L$ from the whole dataset at each iteration), which will result in slightly varied privacy bounds.
    \item The utilization of subsampling in the proof adds to the privacy improvement, and is also closer to actual experimental settings. This tighter guarantee allows for some space to reduce the variance of the added Gaussian noise, which decreases privacy (as noted in the first remark), but increases the model convergence accuracy (since the noise' variance is a major factor sacrificing  accuracy in private optimization algorithms).
\end{itemize}
Please refer to the appendix for the  proof. 

\section{Experiments}
In this section, we present experimental results on the CIFAR-10 dataset. The task is image classification. Our goal is to search a highly-performing neural architecture for this task. Following~\citep{liu2018darts}, we first search an architecture cell by maximizing the validation performance. Given the searched cell, we perform augmentation: the cell is used to compose a larger architecture, which is then trained from scratch and measured on the test set.


\subsection{Experimental Setup}

The search space is the same as that in \citep{liu2018darts}. The candidate operations include: $3\times 3$ and $5\times 5$ separable convolutions, $3\times 3$ and $5\times 5$ dilated separable convolutions, $3\times 3$ max pooling, $3\times 3$ average pooling, identity, and zero. The network is a stack of multiple cells, each consisting of 7 nodes. The CIFAR-10 dataset has 60000 images from 10 classes, 50000 for training and 10000 for testing. During architecture search, we used 25000 images of the training set for validation. During augmentation, all 50000 images in the training set were used for training the composed architecture.  
The variance of noises added to gradient updates of $A$ and $W$ were both set to 1. The hyperparameters $R^G$ and $R^H$ in gradient clipping were set to 0.01 and 0.1 respectively. 
We experiment with the following settings:
\begin{itemize}[leftmargin=*]
    \item NAS with a single party. The vanilla NAS is performed by a single party which has access to all training and validation data.
    \item Federated NAS with $N$ parties, where $N=2,4,8$. The training data is randomly split into $N$ partitions, each held by one party. So is the validation data. The final architecture is evaluated on the test dataset accessible by the server. The gradients calculated by each party are not obfuscated with random noise.
     \item Differentially-private FNAS with $N$ parties, where $N=2,4,8$. The gradients calculated by each party are obfuscated with random noise. The rest of settings are the same as those in FNAS. 
\end{itemize}

\subsection{Results} 

\begin{table*}[t]\centering
\caption{Test error under different settings. Note that the search cost is only about architecture search, not including augmentation which trains the composed  architecture from scratch.}
\label{results}
\vskip 0.15in
\begin{tabular}{l|cccccc}
\toprule
&\multirow{2}{*}{\textbf{\#parties}} & \textbf{Test error} & \textbf{Params} & \textbf{Search cost} & \multirow{2}{*}{\textbf{\#ops}}  \\
& & \textbf{(\%)} & \textbf{(M)} & \textbf{(GPU days)} &    \\
 \midrule
Vanilla NAS & 1  & 2.8 $\pm$ 0.10 & 3.36 & 1.25 & 4\\
 \midrule
\multirow{3}{*}{FNAS} & 2  &2.9 $\pm$ 0.15 & 3.36 & 1.21 & 4\\
 &  4  & 3.2 $\pm$ 0.34 & 3.36 & 0.67 & 4\\
  &  8  & 3.3 $\pm$ 0.40  & 3.36 &0.55 & 4\\
\midrule
\multirow{4}{*}{DP-FNAS}& 1 & 3.0 $\pm$ 0.10 & 3.36 & 1.39 & 4\\
&  2  & 3.0 $\pm$ 0.12& 3.36 & 1.28 & 4\\
&  4  & 3.1 $\pm$ 0.13 & 3.36 & 0.93 & 4\\
&  8  & 3.4 $\pm$ 0.38 & 3.36 & 0.59 & 4\\
\bottomrule
\end{tabular}
\vskip -0.1in
\end{table*}

\begin{table}[h]
\centering
\caption{Validation error achieved by DP-FNAS under different variance of noises. The number of parties is 4.}
\vskip 0.15in
\begin{tabular}{c|c}
\toprule
  \textbf{Variance} & \textbf{Validation } \\
                  \textbf{of Noise} & \textbf{error (\%)}       \\
\midrule
  0.5  & 14.0 $\pm$ 0.32  \\
                             1.0  & 14.0 $\pm$ 0.32  \\
                             2.0  & 14.4 $\pm$ 0.43  \\
                             5.0  & 15.1 $\pm$ 0.85  \\
                             8.0  & 16.4 $\pm$ 1.01  \\
                             10.0 & 19.2 $\pm$ 3.27  \\
\bottomrule
\end{tabular}
\label{variance}
\end{table}

Table~\ref{results} shows the test error and search cost (measured by GPU days) under different settings. From this table, we make the following observations. First, the performance of DP-FNAS with different numbers of parties is on par with that of single-party vanilla NAS. This demonstrates that DP-FNAS are able to search highly-performing neural architectures that are as good as those searched by a single machine while preserving differential privacy of individual parties. Second, in DP-FNAS, as the number of parties increases, the performance drops slightly. This is probably because: Gaussian noise is added to the gradient of each party; more parties result in more added noise, which hurts the convergence of the algorithm. Third, under the same number of parties, DP-FNAS works slightly worse than FNAS. This is because FNAS is noise-free while the gradients in DP-FNAS are obfuscated with noise. However, the performance difference is very small. This shows that DP-FNAS is able to provide stronger privacy protection without substantially degrading performance. Fourth, in FNAS, as the number of parties increases, the performance becomes slightly worse. The possible reason is: as the number of parties increases, the size of data held by each party decreases. Accordingly, the gradient calculated by each party using its hosted data is biased to the data of this party. Such bias degrades the quality of model updates.
Fifth, as the number of parties increases, the search cost decreases. This is not surprising since more parties can contribute more computing resources. However, the rate of cost reduction is not linear in the number of parties. This is because communication between parties incurs latency. Sixth, under the same number of parties, DP-FNAS has slightly larger search cost than FNAS. This is because adding noise renders the gradient updates less accurate, which slows down convergence. Seventh, the number of parameters and operations remain the same under different parties, with or without noise. This indicates that DP-FNAS and FNAS do not substantially alter the architectures,  compared with those searched by a single machine. 

Table~\ref{variance} shows how the validation error of DP-FNAS with 4 parties varies with the variance of noise. As can be seen, large variance results in larger validation error. This is because noises with larger variance tend to have larger magnitude, which makes the gradient updates less accurate. However, a larger variance implies a stronger level of differential privacy. By tuning the variance of noise, we can explore a spectrum of tradeoffs between strength of privacy protection and classification accuracy.

\section{Related Works}
\paragraph{Federated NAS} There are several works independently conducted in parallel to ours on the topic of federated NAS. In~\citep{he2020fednas}, each client locally performs neural architecture search. The architecture variables of different clients are synchronized to their average periodically. This approach has no convergence guarantees. In our work, different parties collaboratively search for a global architecture by exchanging gradients in each iteration, where the convergence is naturally guaranteed. In \citep{zhu2020real}, a federated algorithm is proposed to search neural architectures based on the evolutionary algorithm (EA), which is computationally heavy. In our work, a gradient-based search algorithm is used, which has lower computational cost. In \citep{xu2020neural}, the search algorithm is based on NetAdapt~\citep{yang2018netadapt}, which adapts a pretrained model to a new hardware platform, where the performance of the searched architecture is limited to that of the pretrained model. In our work, the search is performed in a large search space rather than constrained by a human-designed architecture. 

\paragraph{Federated Learning}
Federated learning (FL) is a decentralized learning paradigm which enables multiple parties to collaboratively train a shared model by leveraging data from different parties while preserving privacy. Please refer to~\citep{li2019federated} for an extensive review. One key issue in FL is how to synchronize the different parameter copies among parties. One common approach is periodically setting different copies to their average~\citep{mcmahan2016communication}, which however has no convergence guarantees. Client-server-based architectures guarantee convergence by exchanging gradients and models between servers and clients, but incur high communication overhead. Kone{\v{c}}n{\`y} et al.~\citep{konevcny2016federated} proposed two ways to reduce communication costs: learning updates from a restricted space
parametrized using a smaller number of variables and compressing updates  using  quantization, random rotations, and subsampling.



\paragraph{Neural Architecture Search}
Neural architecture search (NAS) has achieved remarkable progress recently, which aims at searching for the optimal architecture of neural networks to achieve the best predictive performance.  In general, there are three paradigms of methods in NAS: reinforcement learning (RL) approaches~\citep{zoph2016neural,pham2018efficient,zoph2018learning}, evolutionary learning approaches~\citep{liu2017hierarchical,real2019regularized}, and gradient-based approaches~\citep{cai2018proxylessnas,liu2018darts,xie2018snas}. In RL-based approaches, a policy is learned to iteratively generate new architectures by maximizing a reward which is the accuracy on the validation set. Evolutionary learning approaches represent the architectures as individuals in a population. Individuals with high fitness scores (validation accuracy) have the privilege to generate offspring, which replaces individuals with low fitness scores. Gradient-based approaches adopt a network pruning strategy. On top of an over-parameterized network, the weights of connections between nodes are learned using gradient descent. Then weights close to zero are later pruned.

\paragraph{Differential Privacy} Rajkumar and Agarwal~\citep{rajkumar2012differentially} developed differentially-private machine learning algorithms in a distributed multi-party setting. A client-server architecture is used to aggregate gradients computed by individual parties and synchronize different parameter copies. The gradient calculated in each iteration by  each party is added with two sources of random noise: (1) party-dependent and iteration-independent random noise; (2) party-independent and iteration-dependent random noise. Agarwal et al.~\citep{agarwal2018cpsgd} studied distributed stochastic gradient descent algorithms that are both computationally efficient and differentially private. In their algorithm, clients add their share of the noise to their gradients before transmission. Aggregation of gradients
at the server results in an estimate with noise equal to the sum of the noise added at each client. Geyer et al.~\citep{geyer2017differentially} proposed an algorithm for preserving differential privacy on clients' side in federated optimization, by concealing clients' contributions during training and balancing the trade-off between privacy loss and model performance.

\section{Conclusions and Future Works}
In this paper, we study differentially private federated neural architecture search (DP-FNAS), where multiple parties collaboratively search for a  highly-performing neural architecture by leveraging the data from different parties, with strong privacy guarantees. DP-FNAS performs distributed gradient-based optimization of architecture variables and weight parameters  using a parameter server architecture. Gradient updates are obfuscated with random Gaussian noise to achieve differential privacy. We provide theoretical guarantees of DP-FNAS on privacy preservation. Experiments on varying numbers of parties demonstrate that our algorithm can search neural architectures which are as good as those searched on a single machine while preserving privacy of individual parties. For future works, we aim to reduce the communication cost in DP-FNAS, by developing methods such as gradient compression, periodic updates, diverse example selection, etc.

\appendix
\section{Proof of Theorem \ref{thm:dp}}

Algorithm \ref{fig:cov-2m} has two instances of gradient sharing steps, one for optimizing the weight parameters $W$, and the other for the architecture parameters $A$. The gradient for $W$ is calculated using training data, while that for $A$ is calculated using validation data. These two steps in each iteration include two randomized mechanisms, namely $\mathcal{M}^G (D^{(tr)})$ and $\mathcal{M}^H(D^{(val)})$ which are perturbed gradients w.r.t. to $W$ and $A$ respectively. We leverage the fact that the two mechanisms have query functions which are querying on two different datasets with disjoint data points, i.e., the training set will not contain information about individuals which are part of the validation set and vice versa. This limits the association of privacy risk for any  individual with only one of the two datasets. Also we know that composition is concerned with a sequence of analysis on the same dataset where each analysis is informed by the exploration of prior analysis. Hence, composition of these two mechanisms over each iteration will not affect the privacy bounds of each other. In that sense, the compositions \ref{1} and \ref{2} decouple as \ref{3} and \ref{4} respectively for any party $k$ as shown:
\begin{align}
    G_k: &\mathcal{M}_t^{G_k} (D_k^{(tr)}, W[\mathcal{M}_{t-1}^{G_k} (D_k^{(tr)})], A[\mathcal{M}_{t-1}^{H_k}(D_k^{(val)})]) \label{1}\\
    H_k: &\mathcal{M}_t^{H_k} (D_k^{(val)}, W[\mathcal{M}_t^{G_k} (D_k^{(tr)})], A[\mathcal{M}_{t-1}^{H_k}(D_k^{(val)})]) \label{2}\\
    G_k: &\mathcal{M}_t^{G_k} (D_k^{(tr)}, W[\mathcal{M}_{t-1}^{G_k} (D_k^{(tr)})]) \label{3}\\
    H_k: &\mathcal{M}_t^{H_k} (D_k^{(val)}, A[\mathcal{M}_{t-1}^{H_k}(D_k^{(val)})]) \label{4}
\end{align}
where $\mathcal{M}_t^{G_k}$ represents a randomized mechanism for gradient w.r.t. $W$  at the $t^{th}$ iteration for a party $k$.  It takes previous mechanisms ($\mathcal{M}_{t-1}^{G_k}$ via $W$ and $\mathcal{M}_{t-1}^{H_k}$ via $A$) as inputs. Similarly, $\mathcal{M}_t^{H_k}$ represents a randomized mechanism for gradient w.r.t. $A$  at the $t^{th}$ iteration for a party $k$. The above expression is to suggest the recursive phenomena as also evident from Algorithm \ref{fig:cov-2m}. With these in place, we can argue that the two mechanisms are composing independently along the direction of the iterations for each party. (Note that we ignored the presence of validation set ($D_k^{(val)}$) in the same way we ignore that of datasets from other parties ($D_{l \neq k}^{(tr)}$) since in both scenarios the datasets are presumably disjoint to $D_k^{(tr)}$.)

Note that adding or removing one individual would change the value of $\sum_{i \in I_{t}} \bar{g}_{t}^{(i)}$ or $\sum_{i \in I_{t}} \bar{h}_{t}^{(i)}$ (from Algorithm \ref{fig:cov-2m})  by at most  $R^G$ or $R^H$ (clipping constants) in the $l_2$ norm due to the clipping operation. Hence the query function for mechanisms $\mathcal{M}^G (D^{(tr)})$ and $\mathcal{M}^H(D^{(val)})$ has sensitivity $R^G$ and $R^H$ respectively. The major role played by clipping constants reflects in the accuracy achieved by the algorithm. We also subsample the dataset for computing gradients at both instances. We perform Poisson subsampling by choosing a data point with probability $p$ for making a place in the mini-batch used for gradient computation. This gives us the subsampled randomized mechanisms $\mathcal{M}_t^{G_k} (D_k^{(tr)}) \circ Sample_p(D_k^{(tr)})$ and $\mathcal{M}_t^{H^k}(D_k^{(val)})\circ Sample_p(D_k^{(val)})$ similar to the one in \cite{bu2019deep}. The above analysis has translated our problem into two instances of the problem in \cite{bu2019deep}. This allows us to leverage the results from \cite{bu2019deep} for each of these compositions independently, which completes the proof of Theorem~\ref{thm:dp}.

\section{Detailed Proof}
Let $f = T(P, Q)$. This trade-off function says that $f(\alpha)$ is the least accumulated type II error among all tests at significance level $\alpha$. Hence, the larger the value of the trade-off function is, the harder is the hypothesis testing problem, the tighter is the privacy achieved by the private algorithm. With this understanding we define $f$-DP as follows,
\begin{defin} A (randomized) algorithm M is $f$-differentially private if:
\begin{align*}
    T\left(M(S), M\left(S^{\prime}\right)\right) \geqslant f
\end{align*}
for all neighboring datasets $S$ and $S'$
\end{defin} 
In the above definition, the inequality holds pointwise for all $0\leq \alpha \leq 1$. We exploit the notation by identifying $\mathcal{M}(S)$ and $\mathcal{M}(S')$ with their associated probability distributions for the sake of notational ease.\\
We also note the following relation of $f$-DP with $(\epsilon, \delta)$-DP from \cite{wasserman2008statistical},
\begin{defin} (\cite{wasserman2008statistical}) $(\epsilon, \delta)$-DP is a specific instance of $f$-DP as an algorithm is $(\epsilon, \delta)$-DP iff it is $f_{(\epsilon, \delta)}$-DP with (for all $0\leq \alpha \leq 1$)
\begin{align*}
    f_{\epsilon, \delta}(\alpha)=\max \left\{0,1-\delta-\mathrm{e}^{\epsilon} \alpha, \mathrm{e}^{-\epsilon}(1-\delta-\alpha)\right\}
\end{align*}
\end{defin} 
We use the following information as definitions from \citep{bu2019deep},
\begin{defin} \label{sens}
Consider a univariate statistic $\theta (S)$, released privately. The Gaussian mechanism adds a normal random variable $\sim\mathcal{N} (0, \sigma^2)$ as noise to the statistic $\theta$,
which gives $\mu$-GDP if $\sigma = \Delta(\theta)/\mu$. Here $\Delta(\theta)$ represents the sensitivity of $\theta$, defined as $\Delta(\theta) = sup_{S,S_0} |\theta (S)-\theta (S_0)|$, where the supremum is over all neighboring datasets.
\end{defin}
\begin{defin} (Binary Function)
Given the trade-off functions $f=T(P, Q) \text { and } g=T\left(P^{\prime}, Q^{\prime}\right)$, the binary function is defined as $f \otimes g=T\left(P \times P^{\prime}, Q \times Q^{\prime}\right)$
\end{defin}
Due to a central limit theorem phenomenon, the privacy leakage during the composition of many 'very private' $f$-DP algorithms accumulate as $G_\mu$ for some $\mu$ under composition. The formal composition theorem is the following,

\begin{lem} ($f$-DP composition theorem) Assuming each $f_t$ (corresponding to privacy bound of $t$-th iteration of an algorithm)
is very close to $Id(\alpha) = 1- \alpha$ (which indicates perfect privacy),
\cite{bu2019deep} show that
\begin{align*}
    f_{1} \otimes f_{2} \otimes \cdots \otimes f_{T} \text { is approximately } G_{\mu}
\end{align*}
when T is very large, where $\otimes$ is a binary function.
\end{lem}
This compositional privacy bound $f_{1} \otimes f_{2} \otimes \cdots \otimes f_{T}$ cannot be improved in general. \\ \\
According to Poison subsampling, any point in the dataset $S$ makes to the subsample independently with probability p. Let the resulting subsampled data be $Sample_p(S)$. Given any algorithm $\mathcal{M}$, a subsampled algorithm is denoted as $\mathcal{M} \circ Sample_p(S)$.

\begin{lem}($f$-DP Subsampling theorem)  Let a given algorithm $\mathcal{M}$ be $f$-DP, we write $f_p$ for $pf +(1-p)Id$,
and denote $f^{-1} = T(Q, P)$  (if $f = T(P, Q)$), it is shown in \cite{bu2019deep} that the subsampled algorithm $M \circ Sample_{p}$ is $min\{f_p, f_p^{-1}\}^{**}$-DP, where $^{**}$ represents double conjugate.
\end{lem}
The privacy bound $min\{f_p, f_p^{-1}\}^{**}$ is larger than $f$ (which is desirable) but it cannot be improved in general.
\begin{lem} \label{clt}
According to the central limit theorem, when $p\sqrt{T}\rightarrow \nu$, for a constant $\nu$, where $p$ is the subsampling probability and $T$ represents total number of iterations in the algorithm, then as $T \rightarrow \infty$,
\begin{align*}
    f=\left(p G_{1 / \sigma}+(1-p) \mathrm{Id}\right)^{\otimes T} \rightarrow G_{\mu}
\end{align*}
where $\mu=\nu \sqrt{\mathrm{e}^{1 / \sigma^{2}}-1}$. Here, $^{\otimes T}$ represents the composition over $T$ iterations. The proof is given in \cite{bu2019deep}.
\end{lem}
\begin{thm}
Given an optimization algorithm with a general deep neural network loss function, a Gaussian mechanism with noise variance $R\sigma$, where $R$ is the gradient clipping constant, hence the sensitivity of the mechanism's query function and $\sigma$ is the noise random variable variance. The algorithm along with Possion subsampling ($p$) for gradient computation at each iteration, composed over $T$ iterations achieves the following privacy guarantee,
\begin{align*}
    p \sqrt{T\mathrm{e}^{1 / \sigma^{2}}-1}-GDP
\end{align*}
\end{thm}
\begin{proof} The query function for the Gaussian mechanism $\mathcal{M}$ is gradient of a general neural network loss evaluated w.r.t. any model parameters to be optimized. The sensitivity of the query function is given to be $R$. The standard deviation of the added noises is $R\sigma$. According to definition \ref{sens}, it is ensured that $\mathcal{M}$ is $\frac{1}{\sigma}$-GDP. As per the arguments in the Appendix of \cite{bu2019deep} (using composition and subsampling theorem for $f$-DP), $\mathcal{M}_{t=1-T}$ is $min\{f_p, f_p^{-1}\}^{**}$-DP with $f_p=\left(p G_{1 / \sigma}+(1-p) \mathrm{Id}\right)^{\otimes T}$ (composition over $T$ iterations). Using lemma \ref{clt}, assuming $T$ is large (which is true for general settings of neural network optimization)
\begin{align*}
    \min \left\{f, f^{-1}\right\}^{* *} \approx \min \left\{G_{\mu}, G_{\mu}^{-1}\right\}^{* *}=G_{\mu}^{* *}=G_{\mu}
\end{align*}
Hence, the algorithm with composition of subsampled algorithm $\mathcal{M}_{t} \circ Sample_p(D)$ over $T$ iterations is $p \sqrt{T\mathrm{e}^{1 / \sigma^{2}}-1}$-GDP (from lemma \ref{clt}).
\end{proof}

\bibliography{release}

\end{document}